\documentclass{article} 
\usepackage[final]{colm2025_conference} 

\usepackage{microtype}
\usepackage{hyperref}
\usepackage{url}
\usepackage{booktabs}

\definecolor{darkblue}{rgb}{0, 0, 0.5}
\hypersetup{colorlinks=true, citecolor=darkblue, linkcolor=darkblue, urlcolor=darkblue}

\title{Understanding and Improving Noisy Embedding Techniques in Instruction Finetuning}

\author{
Abhay Yadav \\
Johns Hopkins University \\
Baltimore, MD, USA \\
\texttt{ayadav13@jh.edu}
}

\usepackage[utf8]{inputenc}
\usepackage{booktabs}
\usepackage{amsfonts}
\usepackage{nicefrac}
\usepackage{microtype}
\usepackage{amsthm}
\usepackage{epstopdf}
\usepackage{bm,xspace}
\usepackage{verbatim}
\usepackage{multirow}
\usepackage{calc}
\usepackage{framed}
\usepackage{eqparbox}
\usepackage{xspace}
\usepackage{xr}
\usepackage{xcolor}
\usepackage{empheq}
\usepackage[most]{tcolorbox}
\usepackage{placeins}
\usepackage{siunitx}
\usepackage{float}
\usepackage{algorithm}
\usepackage{algorithmic}
\usepackage{natbib}
\usepackage{subcaption}

\newcommand{\llama}{LLaMA}
\newcommand{\neft}{\texttt{NEFT}}
\newcommand{\neftune}{\texttt{NEFTune}}
\newcommand{\symnoise}{\texttt{SymNoise}}

\newtheorem{innercustomlemma}{Lemma}
\newenvironment{customlemma}[1]{\renewcommand\theinnercustomlemma{#1}\innercustomlemma}{\endinnercustomlemma}
\newtheorem{lemma}{Lemma}
\newcommand{\myparagraph}{\textbf}
\renewcommand{\tiny}{\small}

\begin{document}

\maketitle

\begin{abstract}
Recent advancements in instructional fine-tuning have injected noise into embeddings, with \neftune{}~\citep{jain2023neftune} setting benchmarks using uniform noise. Despite \neftune{}'s \textbf{empirical findings} that uniform noise outperforms Gaussian noise, the reasons for this remain unclear. This paper aims to clarify this by offering a thorough analysis, both theoretical and empirical, indicating comparable performance among these noise types. Additionally, we introduce a new fine-tuning method for language models, utilizing symmetric noise in embeddings. This method aims to enhance the model's function by more stringently regulating its local curvature, demonstrating superior performance over the current method, \neftune{}. When fine-tuning the  \llama{}-2-7B model using Alpaca, standard techniques yield a $29.79$\% score on AlpacaEval. However, our approach, \symnoise{}, increases this score significantly to $69.04$\%, using symmetric noisy embeddings. This is a $6.7$\% improvement over the state-of-the-art method, \neftune{}~($64.69$\%). Furthermore, when tested on various models and stronger baseline instruction datasets, such as Evol-Instruct, ShareGPT, OpenPlatypus, \symnoise{} consistently outperforms \neftune{}.  The current literature, including \neftune{}, has underscored the importance of more in-depth research into the application of noise-based strategies in the fine-tuning of language models. Our approach, \symnoise{}, is another significant step towards this direction, showing notable improvement over the existing state-of-the-art method.
\end{abstract}

\section{Introduction}
For Large Language Models~\citep{Vaswani+2017, devlin2018bert, radford2019language-gpt2, raffel2020exploring-t5, zhang2022opt, touvron2023llama, zhao2023survey-llms} to be effective, their proficiency in executing specific instructions is crucial~\citep{wang2022self-selfinstruct, ouyang2022training-instructGPT, brown2020language-gpt3, chung2022scaling-FLAN}. These models typically begin with training on a vast array of unfiltered web data, after which they undergo a more focused fine-tuning stage using a smaller, selectively chosen collection of instructional data. The fine-tuning stage, centered on instructions, is fundamental in unlocking and controlling the full capabilities of LLMs. The practical value of these models is predominantly dependent on how efficiently we can leverage these concise instructional data sets for optimal performance.


In recent years, noise injection~\citep{nukrai2022text-cap, zang2021noise-pointcloud, akbiyik2020data-noiseimage} has been a focal point in computer vision research, yielding methods that enhance model robustness and accuracy. This strategy has recently been adapted for fine-tuning Large Language Models (LLMs), exemplified by the \neftune{} method~\citep{jain2023neftune}, which applies uniform random noise to improve model performance on diverse datasets. Despite \neftune{}'s efficacy surpassing traditional fine-tuning techniques, the reasons behind its success, particularly against the commonly used Gaussian noise, are not entirely understood. Our work demystifies this by presenting a detailed theoretical and empirical analysis that reveals comparable results between noise types when appropriately scaled. Moreover, we introduce a novel noise injection approach which not only facilitates a more intuitive understanding but also achieves superior empirical results, outperforming \neftune{} and other established fine-tuning methods by a considerable margin.

In particular, our objective is to regularize the curvature of the function learned during training. Curvature regularization has been used in domains such as computer vision~\citep{moosavi2019robustness, lee2023explicit}, graph embedding~\citep{pei2020curvature}, and  deep neural networks~\citep{huh2020curvature}. Specifically, we aim to ensure that the function's response changes gradually when the input is modified slightly by noise. In more technical terms, our goal is to have the gradient approach zero in the immediate vicinity of an input altered by a minimal amount. This represents a more stringent condition than merely requiring small values for the Hessian or gradient. However, considering computational efficiency, we opt to avoid the direct computation of gradients or Hessians. Instead, we employ this stringent condition, which, as our experiments on real-world benchmark datasets demonstrate, is effective in practical scenarios.

In this paper, we unveil Symmetric Noise Fine Tuning (\symnoise{}), a new technique that leverages symmetric Bernoulli distribution-based noise applied to the embedding vectors of training data during the finetuning stage. Each noise component is generated with an equal probability of $\frac{1}{2}$ for the values $-1$ and $1$. This method significantly enhances instruction finetuning outcomes, often with remarkable gains, while avoiding additional computational or data resources. While maintaining simplicity, \symnoise{} has a profound impact on downstream conversational output quality. We show that when a large langudage model like \llama{}-2-7B~\citep{touvron2023llama-2} is finetuned using \symnoise{}, its performance on \texttt{AlpacaEval}~\citep{dubois2023alpacafarm-rlhf} rises from $29.79$\% to $69.04$\% – a substantial increase of about $39.25$ percentage points.

Importantly, when compared to the existing \neftune{} method~(which uses random uniform noise), \symnoise{} demonstrates a superior performance edge, outperforming \neftune{} by approximately $6.7$\%. Thus, \symnoise{} not only represents a valuable advancement over traditional finetuning methods but also establishes a new benchmark in efficiency and effectiveness for LLM finetuning.

        \myparagraph{Contributions.} In our comprehensive study, we conduct a detailed theoretical and empirical examination, demonstrating that Gaussian and uniform random noise exhibit functional equivalence when adjusted with an appropriate scaling factor, leading to similar performance on real-world datasets. This insight holds significant importance, especially considering that the creators of the \neftune{} method, a leading approach employing uniform noise, have openly recognized gaps in their understanding of the method's superior performance, notably in comparison to the extensively studied Gaussian noise. By establishing a connection with Gaussian noise, our study helps demystify the \neftune{} method. Moreover, we introduce an innovative noise injection method that exceeds the capabilities of \neftune{} and existing alternatives. Our contributions thus propel the momentum for continued exploration in this field.

\begin{table}
\setlength{\tabcolsep}{2pt} 
\caption{\texttt{AlpacaEval} Win Rate against Text-Davinci-003 when applied with \llama{}-2, trained across diverse datasets. \symnoise{} shows an overall improvement throughout all datasets, outperforming \neftune{} on all datasets. The noise scaling factor for the Gaussian distribution is divided by $\sqrt{3}$, resulting in similar performance for both methods.}
\label{tab:LLaMA-2_GPT}
\centering
\begin{scriptsize}
\begin{tabular}{|l|c|c|c|c|c|}
\toprule
            & Alpaca & Evol-Instruct & ShareGPT & OpenPlatypus & Average \\ \midrule
\llama{}-2 7B  & 29.79  & 70.34    & 68.74              & 62.00        & 57.71   \\ 
+\neft{}       & 64.69  & 79.60    & 76.28              & 70.61        & 72.80  \\ 
+Gaussian      & 64.98  & 78.88    & 75.94             & 70.20        & 72.49  \\ 
+\symnoise{}   & \textbf{69.04}  & \textbf{81.38}    & \textbf{78.67}              & \textbf{72.23}        & \textbf{75.33}  \\ \bottomrule
\end{tabular}
\end{scriptsize}
\end{table}
\FloatBarrier 

\begin{table}
\caption{\texttt{AlpacaEval} Win Rate with and without \neftune{}, \symnoise{} on \llama{}-2, \llama{}-1, and OPT on various datasets. \symnoise{} shows improved performance across these datasets and models.}
\label{table-fig2}
\centering
\begin{scriptsize}
\setlength{\tabcolsep}{3pt} 
\begin{tabular}{|p{2.5cm}|S|S|S|S|}
\hline
Method/Dataset & {Alpaca} & {Evol-Instruct} & {OpenPlatypus} & {ShareGPT} \\ \hline
OPT-6.7B            & 41.4 & 52.2 & 45.7 & 53.4 \\ 
+\neftune{}         & 48.7 & 55.5 & 45.8 & 54.3 \\ 
+\symnoise{}        & 50.8 & 57.6 & 46.9 & 55.6 \\ \hline
\hline
LLaMA-1-7B          & 48.5 & 62.3 & 51.2 & 62.9 \\ 
+\neftune{}         & 61.7 & 67.5 & 56.9 & 63.6 \\ 
+\symnoise{}        & 64.0  & 69.8 & 58.5  & 65.4 \\ \hline
\hline
LLaMA-2-7B          & 48.3 & 62.5 & 57.2 & 63.5 \\ 
+\neftune{}         & 62.5 & 67.6 & 61.7 & 64.2 \\ 
+\symnoise{}        & 64.9 & 69.6 & 62.1 & 66.1 \\ 
\hline
\end{tabular}
\end{scriptsize}
\end{table}
\section{Background and Related Work}
In the evolving landscape of instruction finetuning for Large Language Models (LLMs), initial efforts like FLAN and T0 marked the beginning of cross-task generalization~\citep{sanh2021multitask-T0, wei2021finetuned-FLAN}. These models, involving encoder-decoder language architectures, underwent finetuning across a diverse spectrum of thousands NLP tasks. This progression, detailed in studies by~\citet{chung2022scaling-FLAN} and \citet{xu2022zeroprompt} demonstrated the adaptability of LLMs to a variety of standard NLP tasks.

Following this trajectory, OpenAI's InstructGPT~\citep{ouyang2022training-instructGPT} emerged as a pioneering model adept at handling open-ended questions with remarkable efficiency. This model, an iteration of GPT-3~\citep{brown2020language-gpt3}, incorporated reinforcement learning from human feedback (RLHF), leading to the development of advanced models like ChatGPT~\citep{OpenAI2022ChatGPT}. ChatGPT, in particular, gained widespread attention for generating more coherent and extended texts compared to InstructGPT.

Building on these developments,~\citet{wang2022self-selfinstruct} introduced the Self-Instruct approach, utilizing InstructGPT to generate instructional pairs for further finetuning of foundational models like LLaMA into specialized variants such as Alpaca~\citep{taori2023stanford-alpaca}. Concurrently, the trend towards distilled models, as discussed by~\citet{taori2023stanford-alpaca} and~\citet{xu2023wizardlm}, led to the creation of diverse datasets. These datasets, including works by~\citet{xu2023wizardlm} and~\citet{lee2023platypus}, focused on refining specific model capabilities like STEM question answering and logical reasoning. Another notable advancement was AlpaGasus by~\citet{chen2023alpagasus}, which employed a quality-filtering mechanism based on GPT-4 evaluations to enhance model performance. In a different methodology, ShareGPT, as described by \citet{Chiang2023Vicuna}, was developed through the crowd sourcing of real user interactions sourced from ChatGPT.

In the context of incorporating noise into model training, the pioneering work by~\citet{zhu2019freelb} with the FreeLB method demonstrated the effectiveness of adversarial perturbations in boosting MLM model performance. This method involved introducing calculated Gaussian perturbations into the embeddings and optimizing them to maximally impact model performance. Similar strategies were later applied in various domains, such as image captioning~\citep{nukrai2022text-cap}, point cloud processing~\citep{zang2021noise-pointcloud}, graphs~\citep{kong2022robust-adversarialnoise}, and privacy mechanisms~\citep{dwork2014algorithmic-diffprivacy}. Curvature regularization has been used in domains such as computer vision~\citep{moosavi2019robustness, lee2023explicit}, graph embedding~\citep{pei2020curvature}, and  deep neural networks~\citep{huh2020curvature}. Noise based on the Bernoulli distribution, as opposed to Gaussian or Uniform noise, has been utilized, as mentioned by \citet{spall1998implementation}. In this approach, each outcome, either $-1$ or $1$, is assigned an equal probability of $\frac{1}{2}$.

\section{On Similarity of Uniform noise and Gaussian noise} \label{sec:similarity_uniform_gussian}
In this section, we investigate the similarity between uniform and Gaussian noise when used for embedding perturbations. While these noise types yield different statistical properties in low dimensions, their behavior becomes increasingly similar as the number of dimensions grows. This phenomenon is especially relevant in the context of large language models (LLMs), where embeddings typically reside in high-dimensional spaces.

\begin{lemma}{\normalfont \textbf{(Uniform Distribution } $L_2$ \textbf{ Norm)}} \label{lem:norm_uniform}
    For $P = (x_1, x_2, ..., x_d) \sim U^d([-1, 1])$, the expected $L_2$ norm is:
    \begin{align}
    E[\|P\|_2] = \sqrt{\frac{d}{3}}.
    \end{align}
\end{lemma}
The proof is deferred to Appendix~\ref{sec:uniform_proof}.

\begin{lemma}{\normalfont \textbf{(Gaussian Distribution } $L_2$ \textbf{ Norm)}} \label{lem:norm_gaussian}
    For $P = (x_1, x_2, ..., x_d) \sim N^d(0, 1)$, the expected $L_2$ norm is:
    \begin{align}
    E[\|P\|_2] = \sqrt{d}.
    \end{align}
\end{lemma}
The proof is deferred to Appendix~\ref{sec:normal_proof}.


\begin{figure}[htbp]
    \centering
    \begin{subfigure}[b]{0.40\textwidth}
        \includegraphics[width=\textwidth]{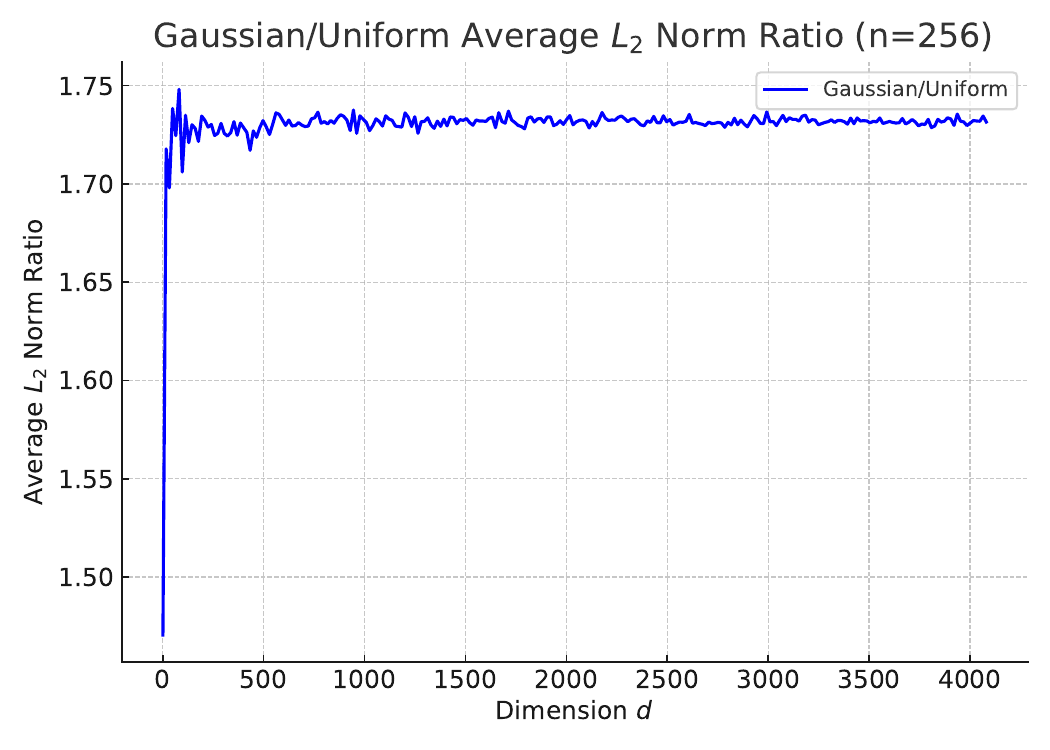}
        \caption{Gaussian/Uniform average $L_2$ ratio}
        \label{fig:gaussian_uniform_ratio_dim}
    \end{subfigure}
    \begin{subfigure}[b]{0.40\textwidth}
        \includegraphics[width=\textwidth]{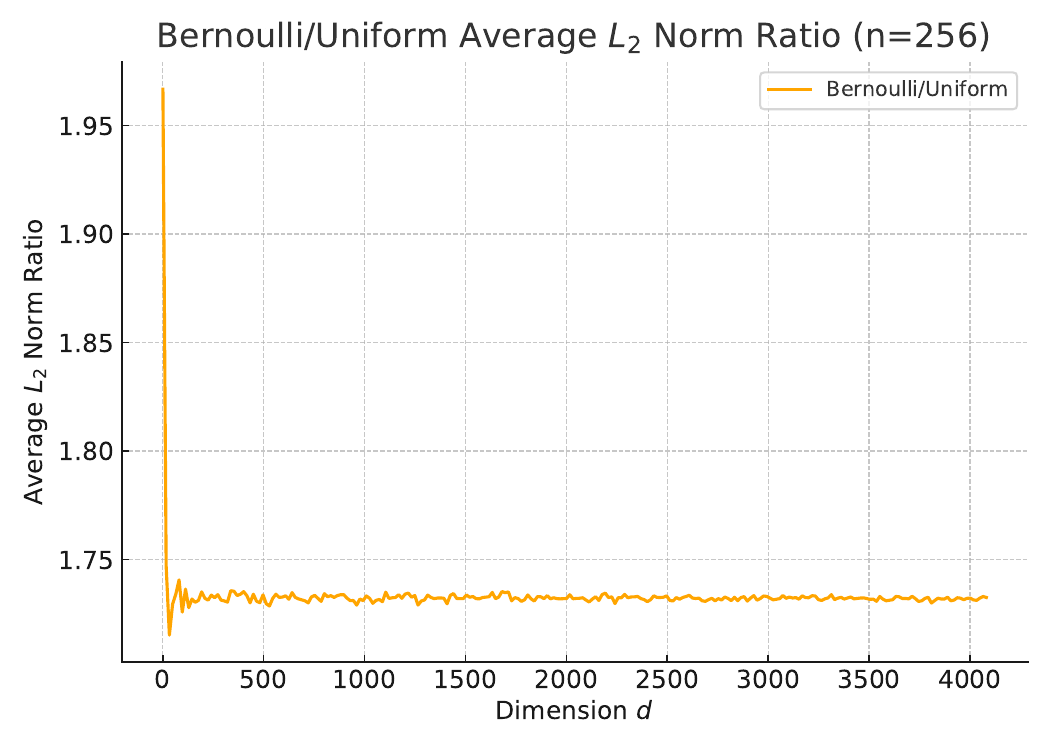}
        \caption{Bernoulli/Uniform average $L_2$ ratio}
        \label{fig:bernoulli_uniform_ratio_dim}
    \end{subfigure}
    \caption{Comparison of average $L_2$ norm ratios for Gaussian and Bernoulli noise relative to Uniform noise as a function of dimensionality.}
    \label{fig:noise_norm_ratios}
\end{figure}

Drawing from Lemma~\ref{lem:norm_uniform} and Lemma~\ref{lem:norm_gaussian}, it is apparent that the expected noise from the Gaussian distribution is $\sqrt{3}$ times that of the Uniform distribution. Consequently, to equate the noise scales for comparison, the noise scaling factor for the Gaussian distribution should be adjusted by a factor of $\sqrt{3}$.

As depicted in Figure~\ref{fig:gaussian_uniform_ratio_dim}, a distinct pattern emerges in the ratio of average noise (quantified via $L_2$ norms) between Gaussian and Uniform distributions as dimensionality increases. Notably, the relative impact of Gaussian noise amplifies, approximating $\sqrt{3}$ times the effect induced by Uniform noise with increasing dimensions. An in-depth exploration and analysis concerning the influence of altering the sample size, while maintaining a fixed dimensionality, are detailed in Appendix~\ref{sec:highdim_analysis}.

Moreover, the comparative results on real-world datasets are presented in Table~\ref{tab:LLaMA-2_GPT}, where all conditions are held constant except for the substitution of the Uniform distribution with a Gaussian distribution. In this context, the noise scaling factor for the Gaussian distribution is adjusted by a factor of $\sqrt{3}$, consistent with the discussion above, and one can notice that the performance of both methods thereafter is comparable. A more detailed ablation study is given in the Sec.~\ref{noise_ablation}.

\section{Proposed Method: \symnoise{}}
In the ideal scenario, our goal is to implement curvature regularization, a technique prevalent in fields such as computer vision~\citep{moosavi2019robustness, lee2023explicit}, graph embedding~\citep{pei2020curvature}, and deep neural networks~\citep{huh2020curvature}. However, due to the high computational cost associated with these methods, we aim to explore an alternative approach that adheres to a more stringent condition. This approach has demonstrated superior performance in practice, surpassing current state-of-the-art methodologies. Specifically, we seek to design a function with a gradient~($\nabla f$) having value as $0$ in the vicinity of the input, i.e., \text{for} $x, \epsilon \in \mathbb{R}^{d}$, $\nabla f = \frac{\left|{\frac{f{\left((x - \epsilon ) \right)} - f{\left(x + \epsilon \right)}}{\epsilon}}\right|}{2} \leq \delta $, when $\delta=0$, we have $f{\left(x + \epsilon  \right)} = f{\left(x - \epsilon  \right)}.$\\
\\
In this formulation, the noise turns out to be based on a Bernoulli distribution, diverging from the more commonly used Gaussian or Uniform noise types. Specifically, it uses values of $-1$ and $1$ with equal probability, as in~\citet{spall1998implementation}, to provide a balanced and predictable effect on the network's learning.


Following \citet{jain2023neftune}, we train instruction-based models using instruction-response pairs. Unlike \neftune{}, which adds uniform noise to token embeddings, we introduce symmetric Bernoulli noise. While we retain the same noise scaling factor, $\epsilon = \alpha / \sqrt{Ld}$ (with $L$ as sequence length, $d$ as embedding dimension, and $\alpha$ as a tunable parameter), our method differs in how noise is applied. Details of our approach, \symnoise{}, are provided in Algorithm~\ref{alg:symnoise}, alongside \neftune{} in Algorithm~\ref{alg:NEFTune} for comparison.

\subsection{On Similarity of Uniform noise and Bernoulli noise} \label{sec:similarity_uniform_burnollie}


\begin{lemma}{\normalfont \textbf{(Bernoulli Distribution } $L_2$ \textbf{ Norm)}} \label{lem:norm_bernoulli}
    For $P = (x_1, x_2, ..., x_d)$, with $x_i \in \{-1, 1\}$ and $P(x_i = 1) = P(x_i = -1) = 0.5$, the expected $L_2$ norm is:
    \begin{align}
    E[\|P\|_2] = \sqrt{d}.
    \end{align}
\end{lemma}
The proof is deferred to Appendix~\ref{sec:burnollie_proof}.


In alignment with the discussions in Sec.~\ref{sec:similarity_uniform_gussian} and corroborated by Lemma\ref{lem:norm_uniform}, Lemma~\ref{lem:norm_gaussian}, and Fig~\ref{fig:bernoulli_uniform_ratio_dim}, it is evident that the noise induced by the Bernoulli distribution is amplified by a factor of $\sqrt{3}$ compared to that of the Uniform distribution. To accommodate this disparity, our proposed method \symnoise{} incorporates this $\sqrt{3}$ scaling factor, as detailed in Algorithm~\ref{alg:symnoise}.

\begin{algorithm}[ht]
   \caption{\neftune{}: \textbf{N}oisy \textbf{E}mbedding Instruction \textbf{F}ine\textbf{tun}ing~(Taken from the paper~\citep{jain2023neftune})} \label{alg:NEFTune}
\vspace{0.2cm}
\begin{algorithmic}
\STATE \textbf{Input:} $\mathcal{D}=\{x_i,y_i\}_1^N$ tokenized dataset, embedding layer $\text{emb}(\cdot)$, rest of model $f_{/\text{emb}}(\cdot)$, \\ model parameters $\theta$, $\text{loss}(\cdot)$, optimizer $\text{opt}(\cdot)$
\STATE Hyperparameter: base noise scale $\alpha \in \mathbb{R^+}$
\STATE Initialize $\theta$ from a pretrained model.
\vspace{0.2cm}
\REPEAT
\STATE $(X_i,Y_i) \sim \mathcal{D}$
\COMMENT{sample a minibatch of data and labels}
\STATE $X_{\text{emb}} \gets \text{emb}(X_i), \mathbb{R}^{B\times L \times d}$
\COMMENT{batch size $B$, seq. length $L$, embedding dimension $d$}
\STATE $\epsilon \sim \text{Uniform}(-1,1), \mathbb{R}^{B\times L \times d}$ \COMMENT{sample a noise vector}
\STATE $X_{\text{emb}}' \gets X_{\text{emb}} + (\frac{\alpha}{\sqrt{Ld}}) \epsilon $ \COMMENT{add scaled noise to embeds \footnote{If sequence lengths in a batch are not equivalent, then $L$ is a vector $\in \mathbb{Z}_{>0}^{B}$ and the scaling factor $(\alpha/\sqrt{Ld})$ is computed independently for each sequence in batch.}}
\STATE $\hat{Y}_i \gets f_{/\text{emb}}(X_{\text{emb}}')$ \COMMENT{make prediction at noised embeddings}
\STATE $\theta \gets \text{opt}(\theta, \text{loss}(\hat{Y}_i,Y_i)) $ \COMMENT{train step, e.g., grad descent}
\UNTIL{Stopping criteria met/max iterations.}
\end{algorithmic}
\end{algorithm}

\begin{algorithm}[ht]
   \caption{\symnoise{}: \textbf{Sym}metric \textbf{Nois}y \textbf{E}mbedding Instruction Finetuning~(Proposed Method)} \label{alg:symnoise}
\vspace{0.2cm}
\begin{algorithmic}
\STATE \textbf{Input:} $\mathcal{D}=\{x_i,y_i\}_1^N$ tokenized dataset, embedding layer $\text{emb}(\cdot)$, rest of model $f_{/\text{emb}}(\cdot)$, \\ model parameters $\theta$, $\text{loss}(\cdot)$, optimizer $\text{opt}(\cdot)$
\STATE  Hyperparameter: base noise scale $\alpha \in \mathbb{R^+}$
\STATE Initialize $\theta$ from a pretrained model.
\vspace{0.2cm}
\REPEAT
\STATE $(X_i,Y_i) \sim \mathcal{D}$
\COMMENT{sample a minibatch of data and labels}
\STATE $X_{\text{emb}} \gets \text{emb}(X_i), \mathbb{R}^{B\times L \times d}$
\COMMENT{batch size $B$, seq. length $L$, embedding dimension $d$}
\STATE $\epsilon \sim \text{Bernoulli}\{-1,1\}, \mathbb{R}^{B\times L \times d}$ \COMMENT{sample a noise vector}
\STATE $X_{\text{emb}}' \gets X_{\text{emb}} + (\frac{\alpha}{\sqrt{Ld}}) \frac{\epsilon}{\sqrt{3}} $ \COMMENT{add scaled noise to embeds \footnote{If sequence lengths in a batch are not equivalent, then $L$ is a vector $\in \mathbb{Z}_{>0}^{B}$ and the scaling factor $(\alpha/\sqrt{Ld})$ is computed independently for each sequence in batch.}}
\STATE $X_{\text{emb}}'' \gets X_{\text{emb}} - (\frac{\alpha}{\sqrt{Ld}}) \frac{\epsilon}{\sqrt{3}} $ \COMMENT{subtract same symmetric noise from embeds}

\STATE $\hat{Y}_i \gets f_{/\text{emb}}(concat(X_{\text{emb}}', X_{\text{emb}}''))$ \COMMENT{make prediction at noised embeddings}
\STATE $\theta \gets \text{opt}(\theta, \text{loss}(\hat{Y}_i,Y_i)) $ \COMMENT{train step}
\UNTIL{Stopping criteria met/max iterations.}
\vspace{0.1cm}
\end{algorithmic}
\end{algorithm}

\section{Experiments}

In this section, we perform numerous experiments across various models and benchmarks to demonstrate the efficacy of our proposed method~\symnoise{} and compare it with existing approaches including~\neft{}.

\subsection{Datasets}
In this section, we delve into finetuning datasets that have either gained widespread popularity or have recently achieved state-of-the-art results. Due to the memory limitations of our hardware setup, our focus is exclusively on single-turn datasets following similar protocol as used in~\citet{jain2023neftune}. The chosen datasets are:
Alpaca~\citep{taori2023stanford-alpaca}, ShareGPT~\citep{Chiang2023Vicuna}, Evol-Instruc~\citep{xu2023wizardlm}, and Open-Platypus~\citep{lee2023platypus}. More details about these datasets are in Appendix~\ref{app:data_sets}



In the fine-tuning phase, each model, with the exception of ShareGPT, utilizes the prompt from the Alpaca system. Conversely, ShareGPT is fine-tuned using the prompt from the Vicuna system. Our approach to hyperparameter tuning, including the selection of values, aligns with the methodologies suggested by \citet{jain2023neftune}. We adhered strictly to the same set of hyperparameters as those employed in \neftune{}~\citep{jain2023neftune}.

\subsection{Models}
Following~\citet{jain2023neftune} setup for experimentation, our experiments predominantly utilize Large Language Models (LLMs) with a parameter size of $7$ billion. Specifically, our focus is on models such as \llama{}-1~\citep{touvron2023llama-1}, \llama{}-2~\citep{touvron2023llama-2}, and OPT-6.7B~\citep{zhang2022opt}. These transformer-based models primarily differ in the amount of training data they've been exposed to, with OPT-6.7B, \llama{}-1, and \llama{}-2 being trained on $180$ billion, $1$ trillion, and $2$ trillion tokens, respectively. This variance in training data volume is expected to manifest in their performance across benchmarks like MMLU, where \llama{}-2 typically outperforms the others. 

\subsection{Evaluation Protocols}
Our experimental framework, adapted from the original \neftune{}~\citep{jain2023neftune} setup, primarily utilizes single-turn data for training. We assess the models' conversational skills using \texttt{AlpacaEval} and examine their performance on tasks from the OpenLLM Leaderboard. This was done to verify that the introduction of our \texttt{symnoise} augmentation does not negatively impact the models' performance on standard multiple-choice tasks. Notably, the results demonstrate that our augmented models consistently outperform the original \texttt{neftune} models, albeit by a modest margin.

\begin{itemize}
    \item \textbf{AlpacaEval}: Introduced by~\citet{dubois2023alpacafarm-rlhf}, \texttt{AlpacaEval} is crucial for appraising generative quality. It functions as an automated model-based evaluator, comparing Text-Davinci-003's generations with our model's over $805$ prompts, focusing on the \textit{Win Rate}. The Win Rate indicates how often our model is preferred over Text-Davinci-003, as judged by model evaluator~(GPT-4, ChatGPT etc). The dataset's $805$ test prompts, sourced from various platforms, ensure a comprehensive testing scope. \texttt{AlpacaEval's} high human agreement rate~\citep{dubois2023alpacafarm-rlhf}, validated by 20K annotations, highlights its usefulness and accuracy. We employ both GPT-4 and ChatGPT as evaluators, using ChatGPT initially due to GPT-4's API limitations and costs.

    \item \textbf{Hugging Face OpenLLM Leaderboard}: For leaderboard assessments, datasets like ARC~\citep{clark2018think-arc}, HellaSwag~\citep{zellers2019hellaswag-swag}, and MMLU~\citep{hendrycks2020measuring-mulu} are utilized. These verbalized multiclass classification datasets test the LLM's capability in factual questioning and reasoning. Our evaluations confirm that the \symnoise{} method does not diminish the models' proficiency in these domains.
\end{itemize}

\subsection{Results}
The methodology we employed for tuning hyperparameters and choosing their values adheres closely to the protocols proposed by \citet{jain2023neftune}. Specifically, we meticulously adopted the identical hyperparameter set as delineated in \neftune{} by \citet{jain2023neftune}.
\subsubsection{Improvement in generated text quality}
Our results demonstrate an enhanced metric performance compared to~\neftune{} in terms of generated text quality. As evident from Table~\ref{tab:LLaMA-2_GPT}, there is a notable improvement across all datasets at the 7B size, with an average increase of $17.6\%$ (compared to~\neftune{}'s improvement of $15.1\%$). This indicates that the implementation of \symnoise{} significantly enhances conversational capabilities and answer quality. These findings are supported by evaluations using \texttt{AlpacaEval}, where \symnoise{} notably outperforms \neftune{}. Furthermore, as shown in Table \ref{table-fig2}, enhancements are also observed in older models like \llama{}-1 and OPT-6.7B, with \symnoise{} consistently surpassing \neftune{} in these models as well. An interesting observation is the comparatively smaller gain by \neftune{} in ShareGPT, as per ChatGPT's analysis, a trend not mirrored in GPT-4's evaluation. However, \symnoise{} consistently excels over \neftune{} for ShareGPT in evaluations by both GPT-4 and ChatGPT. In Table \ref{tab:LLaMA-2_GPT}, the Win Rate shows a significant increase from $29.79\%$ to $69.04\%$ for Alpaca, thereby outperforming the state-of-the-art method \neftune{} by $6.7\%$.

\subsubsection{Improvement in \textit{OpenLLM Leaderboard} tasks}
In addressing the potential that \symnoise{} could enhance conversational abilities at the expense of traditional skills, we conducted evaluations using tasks from the OpenLLM Leaderboard. Employing the LM-Eval Harness framework \citep{gao2021framework-lm-eval-harness}, we assessed our model's performance on benchmarks such as MMLU~\citep{hendrycks2020measuring-mulu}, ARC~\citep{clark2018think-arc}, and HellaSwag~\citep{zellers2019hellaswag-swag}. These tests shed light on the model’s knowledge base, reasoning capabilities, and adherence to factual information. As illustrated in Figure~\ref{tab:fig3_openllm}, the results indicate that \symnoise{} not only stabilizes scores but also actively preserves and, in some cases, enhances the model's capabilities. Notably, \symnoise{} consistently outperforms \neftune{} in terms of performance, highlighting its effectiveness in striking a balance between conversational proficiency and traditional computational skills.

\begin{table}
\centering
\caption{For OpenLLM Leaderboard tasks, the influence of \neftune{} and \symnoise{} is investigated on \llama{}-2, encompassing Alpaca, Evol-Instruct, and OpenPlatypus datasets, alongside \llama{}-1 trained on the Evol-Instruct dataset. Comparative observations reveal a uniformity in performance metrics across the diverse datasets and models, indicating negligible impact of \neftune{} but slightly better performance of \symnoise{} on the overall effectiveness. We follow the similar procedure as mentioned in ~\citet{jain2023neftune}, and report their results for completeness. In order to minimize computational expenses, we refrained from conducting thorough hyper-parameter optimization, which may have further improved the results.}
\label{tab:fig3_openllm}
\begin{tiny}
\begin{tabular}{lccc}
\toprule
Task & Llama-2 7B (Alpaca) & +NEFT & +SymNoise \\
\midrule
ARC & 56.4 & 56.1 & 56.0 \\
HellaSwag & 80.0 & 80.1 & 80.2 \\
MMLU & 47.9 & 47.7 & 47.9 \\
\midrule
& Llama-2 7B (OpenPlatypus) & +NEFT & +SymNoise \\
\midrule
ARC & 54.2 & 55.4 & 55.7 \\
HellaSwag & 80.4 & 80.6 & 80.8 \\
MMLU & 43.9 & 45.3 & 45.5 \\
\midrule
 & Llama-1 7B (Evol-Instruct.) & +NEFT & +SymNoise \\
\midrule
ARC & 53.7 & 54.1 & 56.0 \\
HellaSwag & 77.9 & 78.0 & 78.9 \\
MMLU & 38.3 & 38.4 & 39.1 \\
\bottomrule
\end{tabular}
\end{tiny}
\end{table}

\subsection{Analysis}

As shown in \neftune{}~\citep{jain2023neftune} and related work, adding noise to embeddings during training helps mitigate overfitting to dataset-specific quirks like formatting or phrasing. This shifts the model from memorizing instructions to leveraging the broader capabilities of the pretrained base model. A direct effect is that models produce longer, more coherent responses—preferred by both human and automated evaluators~\citep{dubois2023alpacafarm-rlhf}. While increased verbosity contributes to performance gains, our analysis shows that \symnoise{} improves both response quality and quantity beyond what \neftune{} achieves.

Conceptually, \symnoise{} assigns probability mass to multiple noisy variants of instructions, encouraging the model to learn a broader, more uniform distribution rather than overfitting to the training data or a single perturbed version. This promotes better generalization and reduces overfitting.


\subsubsection{Longer responses vs repetition}\label{length_vs_rep}
In this section, our objective is to determine whether the lengthier responses produced using \symnoise{} are a result of increased repetition or if they contribute to more diverse and detailed content. 

Echoing the insights from \citet{jain2023neftune} and supporting evidence from leaderboard performances, a notable correlation emerges between extended response lengths and improved performance in the \texttt{AlpacaEval} task. This raises the question of whether the augmentation of response length by \symnoise{} could lead to diminished text diversity and quality. Our analysis scrutinized the frequency of N-gram repetitions in responses generated by \llama{}-2, trained on various datasets, both with and without \symnoise{} application.

Following the methodology of \citet{jain2023neftune}, our analysis was restricted to the initial segments of each response to maintain consistency. Specifically, we examined the first $50$ words for Alpaca-trained models, $100$ words for Evol-Instruct, and $150$ words for OpenPlatypus, ensuring that at least half of the responses exceeded these thresholds. Responses shorter than these limits were excluded from the analysis.


As delineated in Table \ref{tab:length_vs_rep}, the findings reveal that \symnoise{} typically yields lengthier responses. However, importantly, the frequency of 2-gram repetitions and overall token log-diversity remain largely consistent, paralleling the results observed with \neftune{}. This suggests that the increased length of responses under \symnoise{} is not simply due to repetitive content, but rather indicates the inclusion of additional, relevant information, thereby enriching the depth and value of the generated responses.

\subsubsection{Ablation study with different strength of noise}\label{noise_ablation}
In this section, we explored the efficacy of employing different noise distributions, specifically uniform~(\neftune{}) versus Gaussian noise, versus within the \symnoise{}{} algorithm. From the Table~\ref{tab:noise_ablation}, one can notice that Gaussian noise tends to produce longer outputs. However, this increased length does not correlate with a corresponding enhancement in performance. While it is generally observed that longer generations are associated with improved scoring, none of the generation-time strategies employed matched the effectiveness of models trained with \symnoise{}. Interestingly, our innovative approach, \symnoise{}, exhibits superior performance, surpassing benchmark results. It demonstrates an approximate improvement of $6.7\%$ over the models utilizing \neftune{}. Furthermore, we conducted a comparative analysis with Bernoulli noise to underscore the effectiveness of the symmetric opposing noise component in \symnoise{}.

Moreover, we maintained consistent experimental conditions while substituting the Uniform distribution with a Gaussian distribution. In alignment with our theoretical framework, we adjusted the noise scaling factor for the Gaussian distribution by dividing it by $\sqrt{3}$. This adjustment led to comparable performance between the two methods across various \neftune{} noise levels, reinforcing the validity of our noise scaling approach.

\begin{table*}[htbp]
\centering
\begin{minipage}[t]{0.35\textwidth}
\centering
\caption{\texttt{AlpacaEval} Win Rate and Average Character Count assessed by ChatGPT across various noise settings.}
\label{tab:noise_ablation}
\begin{scriptsize}
\begin{tabular}{|p{1.5cm}|p{1.5cm}|p{1.5cm}|p{1.5cm}|} \hline
Setting & Alpaca & Evol-Instruct & OpenPlatypus \\ \hline
\llama{}-2-7b & 48.26 (375) & 62.55 (864) & 57.20 (1101) \\ \hline
+\neft{} \\ Noise 5 & 62.55 (1062) & 67.58 (1404) & 60.99 (1428) \\ \hline
+\neft{}  \\ Noise 10 & 61.18 (1010) & 65.59 (1697) & 60.62 (1834) \\ \hline
+\neft{}   \\ Noise 15 & 61.86 (820) & 66.58 (1651) & 61.74 (1694) \\ \hline
+Gaussian Noise 5/$\sqrt{3}$ & 62.6 (1073) & 68.01 (1431) & 60.31 (1437) \\ \hline
+Gaussian Noise 10/$\sqrt{3}$ & 61.01 (1211) & 65.29 (1783) & 60.32 (1878) \\ \hline
+Gaussian Noise 15/$\sqrt{3}$ & 61.93 (835) & 65.99 (1767) & 61.38 (1806) \\ \hline
+Gaussian Noise 5 & 60.93 (1371) & 65.09 (2066) & 59.13 (2061) \\ \hline
+Bernoulli Noise 5 & 61.32 (1272) & 65.10 (1840) & 60.22 (1968) \\ \hline
+\symnoise{} Noise 5 & \textbf{64.92} (1186) & \textbf{69.62} (1700) & \textbf{62.14} (1689) \\ \hline
\end{tabular}
\end{scriptsize}
\end{minipage}
\hfill
\begin{minipage}[t]{0.44\textwidth}
\centering
\caption{Transposed view of average lengths and 2-gram repetition rates in \texttt{AlpacaEval} responses for different training methods.}
\label{tab:length_vs_rep}
\begin{scriptsize}
\begin{tabular}{|l|c|c|c|}
\hline
Metric & \llama{}-2 7B & +\neft{} & +\symnoise{} \\ \hline
\multicolumn{4}{|c|}{\textbf{Character Length}} \\ \hline
Alpaca & 375 & 1061 & 1186 \\
Evol-Instruct & 864 & 1403 & 1700 \\
OpenPlatypus & 1100 & 1694 & 1689 \\ \hline
\multicolumn{4}{|c|}{\textbf{Whitespace Length}} \\ \hline
Alpaca & 60 & 169 & 176 \\
Evol-Instruct & 138 & 225 & 245 \\
OpenPlatypus & 170 & 264 & 270 \\ \hline
\multicolumn{4}{|c|}{\textbf{2-Gram Repetition \%}} \\ \hline
Alpaca & 1.49 & 1.72 & 1.80 \\
Evol-Instruct & 3.87 & 3.79 & 3.80 \\
OpenPlatypus & 2.73 & 3.21 & 3.30 \\ \hline
\end{tabular}
\end{scriptsize}
\end{minipage}
\end{table*}

\section{CONCLUSION}
In this work, we rigorously establish that Gaussian and uniform random noise are functionally analogous, contingent on appropriate scaling, and demonstrate similar effectiveness on real-world datasets. This revelation is pivotal, particularly in light of the \neftune{} creators' admission of the method's unexplained superiority, especially over the well-examined Gaussian noise. This insight not only sheds light on the previously opaque superiority of the \neftune{} method but also bridges the gap with the well-understood Gaussian noise.

Furthermore, we have introduced \symnoise{}, a novel noise injection technique that outperforms \neftune{} and other existing methods by a large margin. The advancements showcased by \symnoise{} in training large language models (LLMs) emphasize the importance of innovative algorithmic strategies and regularization techniques. Echoing the sentiments of \citep{jain2023neftune}, the field of LLMs, unlike its counterpart in computer vision, has often favored standardized training methods focusing on model scaling and dataset expansion. However, \symnoise{} underscores the potential of fine-tuning techniques in enhancing model performance, particularly in situations where overfitting to limited instruction datasets is a concern.


\pagebreak

\bibliography{colm2025_conference}
\bibliographystyle{colm2025_conference}

\appendix
\section{Appendix}

\subsection{Datasets} \label{app:data_sets}
\begin{itemize}
  \item \textbf{Alpaca~\cite{taori2023stanford-alpaca}}: Developed using the Self-Instruct method by~\citet{wang2022self-selfinstruct} and the Text-Davinci-003~\citet{ouyang2022training-instructGPT} model (Ouyang et al., 2022), Alpaca leverages a small set of seed tasks to generate new instruction tuning tasks and filter out ineffective ones. This dataset has been instrumental in advancing instruction-based learning.
  \item \textbf{ShareGPT~\cite{Chiang2023Vicuna}}: Comprising 70k voluntarily shared ChatGPT conversations~\cite{ShareGPT2023}, ShareGPT provides a rich source of real-world interaction data. While originally multi-turn, we adapt it to a single-turn format using the Vicunav1.1 dataset version for consistency with our experimental setup.
  \item \textbf{Evol-Instruc~\cite{xu2023wizardlm}}: This dataset, comprising 70k single-turn instructions, is considered more complex than Alpaca. Originating from the Alpaca dataset, Evol-Instruct employs ChatGPT to refine and evolve the initial instructions, thus broadening the scope and complexity of the tasks.
  \item \textbf{Open-Platypus~\cite{lee2023platypus}}: Formed by combining $11$ open-source datasets, Open-Platypus is tailored to enhance LLM performance in STEM and logical reasoning domains. It includes approximately $25$k questions, with around $10$\% generated by LLMs and the rest by human experts. This dataset emphasizes the importance of variety and complexity in question formats.

\end{itemize}

\subsection{Analysis of Distributional Characteristics in High-Dimensional Spaces} \label{sec:highdim_analysis}

In this section, we analyze the behavior of Gaussian, Bernoulli, and Uniform distributions in high-dimensional spaces. We explore how the average \(L_2\) norm ratios of these distributions change with respect to varying dimensions and sample sizes, providing insights into their geometric properties and implications for high-dimensional data analysis.

\subsubsection{Average \(L_2\) Norm Ratio with Varying Dimensionality}

\begin{figure}[htbp]
\centering
\includegraphics[width=0.8\textwidth]{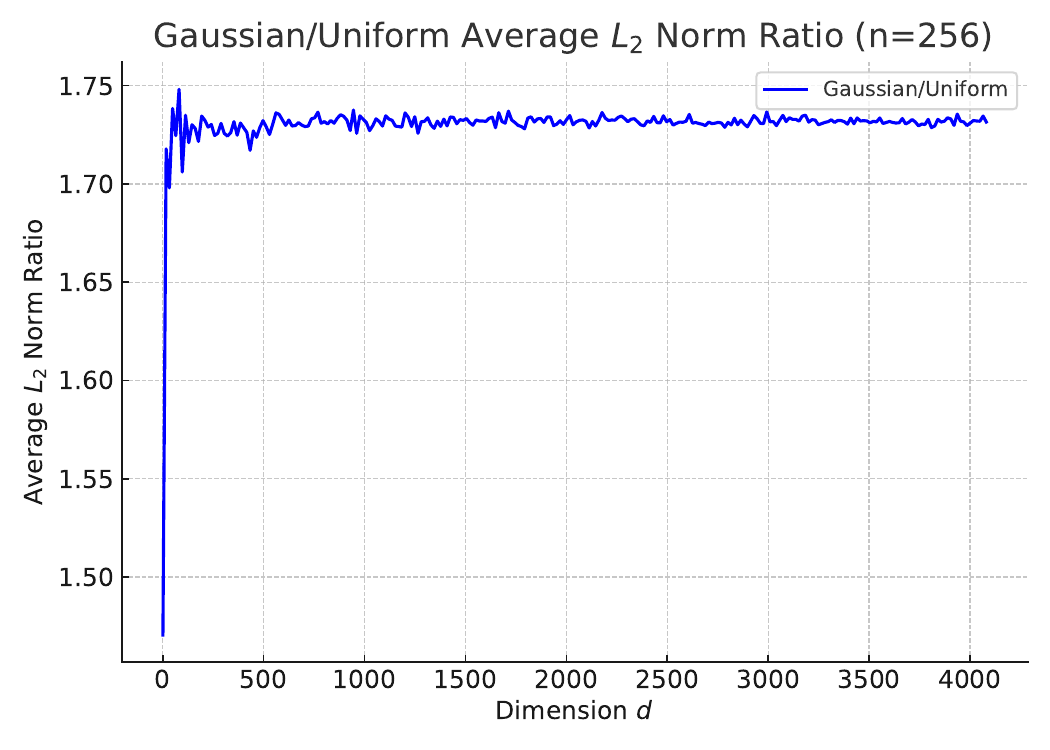}
\caption{Gaussian/Uniform Average \(L_2\) Norm Ratio as a Function of Dimensionality. The plot illustrates the ratio of the average \(L_2\) norm of points drawn from a Gaussian distribution to that of a Uniform distribution, with the number of points fixed at 256 and the dimensionality varying from 1 to 4096.}
\label{fig:gaussian_uniform_ratio_dim_appndx}
\end{figure}


\begin{figure}[htbp]
\centering
\includegraphics[width=0.8\textwidth]{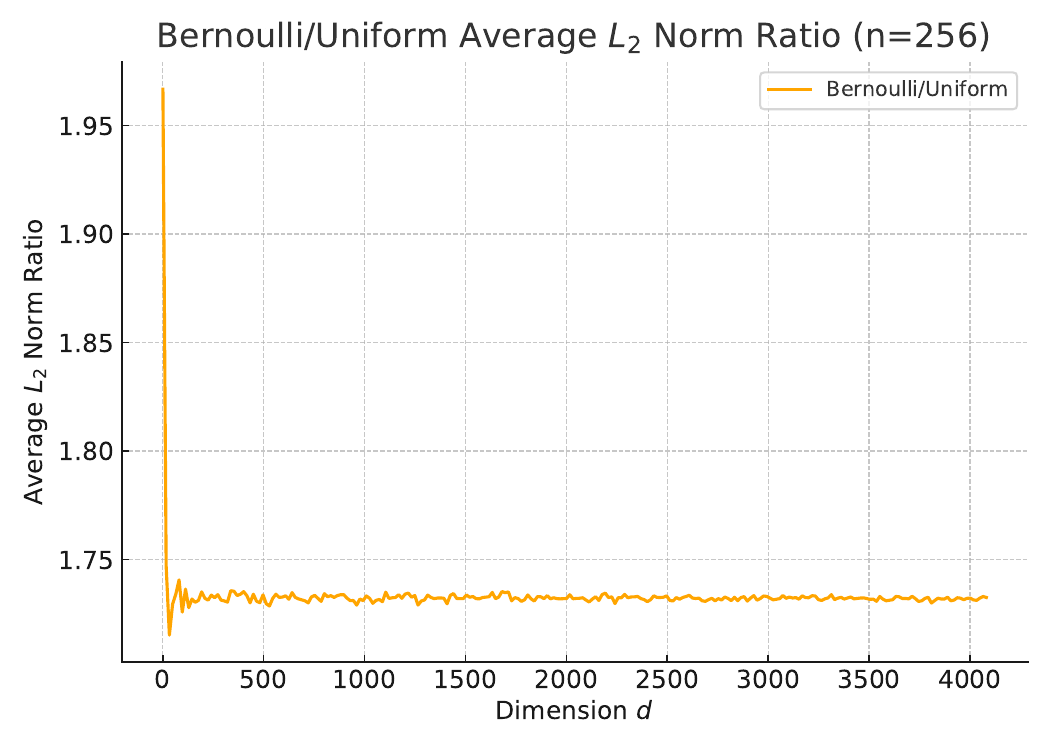}
\caption{Bernoulli/Uniform Average \(L_2\) Norm Ratio as a Function of Dimensionality. The plot depicts the ratio of the average \(L_2\) norm of points drawn from a Bernoulli distribution to that of a Uniform distribution, with the number of points fixed at 256 and the dimensionality varying from 1 to 4096.}
\label{fig:bernoulli_uniform_ratio_dim_appndx}
\end{figure}


\subsubsection{Average \(L_2\) Norm Ratio with Varying Number of Points}

\begin{figure}[htbp]
\centering
\includegraphics[width=0.8\textwidth]{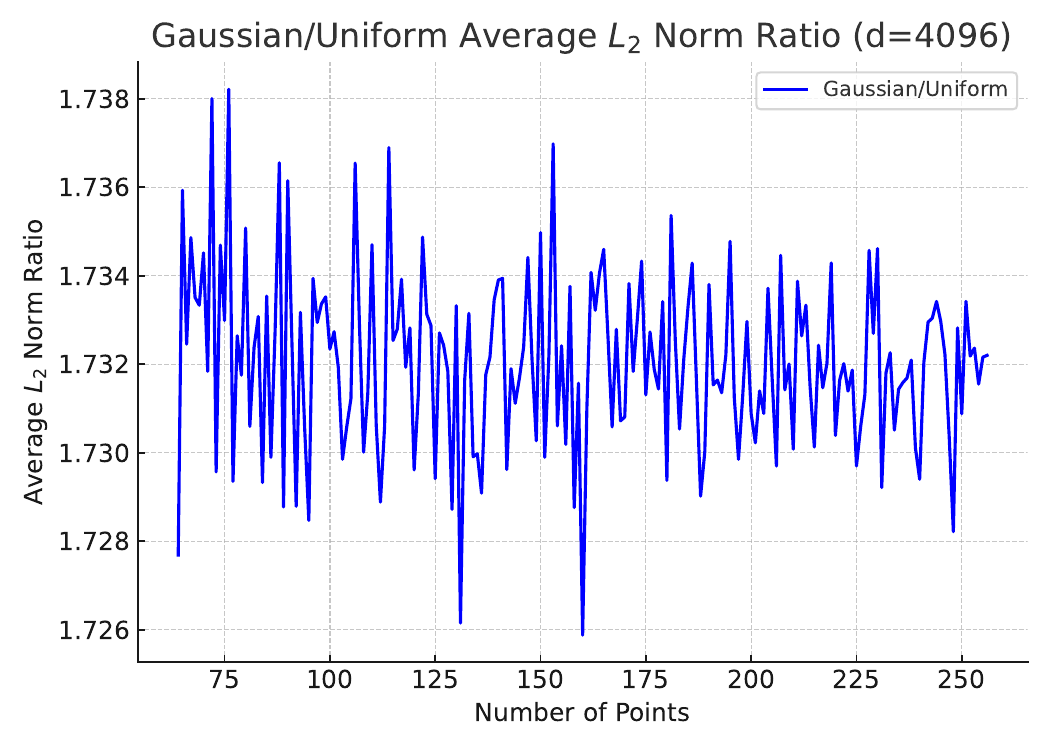}
\caption{Gaussian/Uniform Average \(L_2\) Norm Ratio for Varying Number of Points. The plot illustrates the ratio of the average \(L_2\) norm of points drawn from a Gaussian distribution to that of a Uniform distribution, with the dimensionality fixed at 4096 and the number of points varying from 64 to 256.}
\label{fig:gaussian_uniform_ratio_points_appndx}
\end{figure}


\begin{figure}[htbp]
\centering
\includegraphics[width=0.8\textwidth]{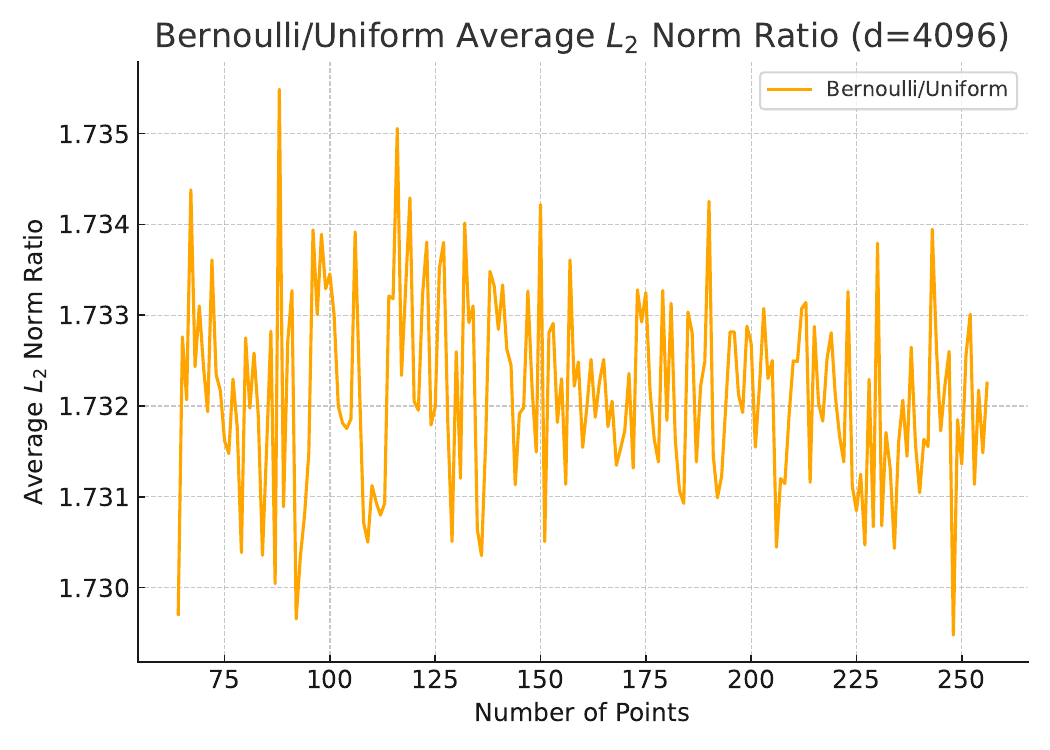}
\caption{Bernoulli/Uniform Average \(L_2\) Norm Ratio for Varying Number of Points. The plot depicts the ratio of the average \(L_2\) norm of points drawn from a Bernoulli distribution to that of a Uniform distribution, with the dimensionality fixed at 4096 and the number of points varying from 64 to 256.}
\label{fig:bernoulli_uniform_ratio_points_apnx}
\end{figure}




\pagebreak
\section{Deferred proofs}
        In this section, we show the proofs omitted from Sec.~\ref{sec:similarity_uniform_gussian} and Sec.~\ref{sec:similarity_uniform_burnollie}.

\subsubsection{Proof of Lemma~\ref{lem:norm_uniform}} \label{sec:uniform_proof}
We state again Lemma~\ref{lem:norm_uniform} from Sec.~\ref{sec:similarity_uniform_gussian} and present the proof.

\begin{customlemma}{\ref{lem:norm_uniform}}{\normalfont \textbf{(Uniform Distribution } $L_2$ \textbf{ Norm)}} \nonumber
For $P = (x_1, x_2, ..., x_d) \sim U^d([-1, 1])$, the expected $L_2$ norm is:
\begin{align}
E[\|P\|_2] = \sqrt{\frac{d}{3}}.
\end{align}
    \textbf{Proof:} 
    Each $x_i$ is uniformly distributed over $[-1, 1]$. The second moment about the origin for a uniform distribution $U(a, b)$ is given by $\frac{b^3 - a^3}{3(b - a)}$. For $U(-1, 1)$, this yields $E[x_i^2] = \frac{1}{3}$. 
    The components $x_i$ are independent, hence the sum of their squares, which represents the $L_2$ norm squared, is the sum of their expected values: $E[\|P\|_2^2] = \sum_{i=1}^{d} E[x_i^2] = d \cdot \frac{1}{3}$. 
    Taking the square root gives the expected $L_2$ norm: $E[\|P\|_2] = \sqrt{E[\|P\|_2^2]} = \sqrt{\frac{d}{3}}$.
\end{customlemma}

\subsubsection{Proof of Lemma~\ref{lem:norm_gaussian}} \label{sec:normal_proof}
We state again Lemma~\ref{lem:norm_gaussian} from Sec.~\ref{sec:similarity_uniform_gussian} and present the proof.

\begin{customlemma}{\ref{lem:norm_gaussian}}{\normalfont \textbf{(Gaussian Distribution } $L_2$ \textbf{ Norm)}} \nonumber
For $P = (x_1, x_2, ..., x_d) \sim N^d(0, 1)$, the expected $L_2$ norm is:
\begin{align}
E[\|P\|_2] = \sqrt{d}.
\end{align}
    \textbf{Proof:} 
    Each $x_i$ is distributed according to $N(0, 1)$. The square of a standard normal variable, $x_i^2$, follows a chi-squared distribution with 1 degree of freedom, for which the mean (expected value) is 1. 
    Given the independence of the components $x_i$, the expected value of the sum of their squares, representing the $L_2$ norm squared, is: $E[\|P\|_2^2] = \sum_{i=1}^{d} E[x_i^2] = d \cdot 1 = d$.
    The expected $L_2$ norm is the square root of this sum: $E[\|P\|_2] = \sqrt{E[\|P\|_2^2]} = \sqrt{d}$.
\end{customlemma}

\subsubsection{Proof of Lemma~\ref{lem:norm_bernoulli}} \label{sec:burnollie_proof}
We state again Lemma~\ref{lem:norm_bernoulli} from Sec.~\ref{sec:similarity_uniform_burnollie} and present the proof.

\begin{customlemma}{\ref{lem:norm_bernoulli}}{\normalfont \textbf{(Bernoulli Distribution } $L_2$ \textbf{ Norm)}} \nonumber
    For $P = (x_1, x_2, ..., x_d)$, with $x_i \in \{-1, 1\}$ and $P(x_i = 1) = P(x_i = -1) = 0.5$, the expected $L_2$ norm is:
\begin{align}
E[\|P\|_2] = \sqrt{d}.
\end{align}
    \textbf{Proof:} 
    Each $x_i$ takes values -1 or 1 with equal probability, leading to $x_i^2 = 1$ irrespective of $x_i$'s actual value. Hence, $E[x_i^2] = 1$. 
    Given the independence of the components $x_i$, the expected value of the sum of their squares, which represents the $L_2$ norm squared, is simply the sum of the expected values: $E[\|P\|_2^2] = \sum_{i=1}^{d} E[x_i^2] = d \cdot 1 = d$. 
    Therefore, the expected $L_2$ norm is the square root of this value: $E[\|P\|_2] = \sqrt{E[\|P\|_2^2]} = \sqrt{d}$.
\end{customlemma}

\end{document}